\documentclass[11pt]{article}
\usepackage{amsmath,amsthm,amssymb,dsfont,stmaryrd}
\usepackage{hyperref}
\usepackage[in]{fullpage}
\usepackage[T1]{fontenc}

\newtheorem{lem}{Lemma}
\newtheorem{thm}[lem]{Theorem}

\newcommand{\R}{{\mathbb R}}

\newcommand{\RR}{{\cal R}}

\renewcommand{\H}{{\cal H}}
\newcommand{\X}{{\cal X}}
\newcommand{\D}{{\cal D}}

\newcommand{\F}{{\cal F}}
\newcommand{\W}{{\cal W}}

\newcommand{\B}{{\cal B}}

\newcommand{\bc}[1]{\left\{{#1}\right\}}
\newcommand{\br}[1]{\left({#1}\right)}
\newcommand{\bs}[1]{\left[{#1}\right]}

\newcommand{\norm}[1]{\left\| {#1} \right\|}

\newcommand{\bsd}[1]{\left\llbracket{#1}\right\rrbracket}

\newcommand{\etal}{\emph{et al}\ }
\renewcommand{\O}[1]{{\cal O}\br{{#1}}}

\newcommand{\Om}[1]{\Omega\br{{#1}}}
\newcommand{\E}[1]{{\mathbb E}\bsd{{#1}}}

\newcommand{\EE}[2]{\underset{#1}{\mathbb E}\bsd{{#2}}}

\newcommand{\ip}[2]{\left\langle{#1},{#2}\right\rangle}

\renewcommand{\vec}[1]{{\mathbf{#1}}}

\newcommand{\vecx}{\vec{x}}

\newcommand{\vecz}{\vec{z}}

\newcommand{\vecw}{\vec{w}}

\newcommand{\vecmu}{{\boldsymbol{\mu}}}
\newcommand{\veckappa}{{\boldsymbol{\kappa}}}

\newcommand{\veczero}{\vec{0}}

\title{Generalization Guarantees for a Binary Classification Framework for Two-Stage Multiple Kernel Learning}
\author{Purushottam Kar\\Department of Computer Science and Engineering\\IIT Kanpur\\\texttt{purushot@cse.iitk.ac.in}}
\date{\today}

\begin{document}
\maketitle

\begin{abstract}
We present generalization bounds for the TS-MKL framework for two stage multiple kernel learning. We also present bounds for sparse kernel learning formulations within the TS-MKL framework.
\end{abstract}

\section{Introduction}
Recently Kumar \etal \cite{two-stage} proposed a framework for two-stage multiple kernel learning that combines the idea of target kernel alignment and the notion of a \emph{good} kernel proposed in \cite{sim} to learn a good Mercer kernel. More specifically, given a finite set of base kernels $K_1,\ldots,K_p$ over some common domain $\X$, we wish to find some combination of these base kernels that is well suited to the learning task at hand. The paper considers learning a positive linear combination of the kernels $K_{\vecmu} = \sum_{i=1}^p\vecmu_iK_i$ for some $\vecmu \in \R^p, \vecmu \geq 0$. It is assumed that the kernels are uniformly bounded i.e. for all $\vecx_1,\vecx_2 \in \X$ and $i = 1 \ldots p$, we have $K_i(\vecx_1,\vecx_2) \leq \kappa_i^2$ for some $\kappa_i > 0$. Let $\veckappa = \br{\kappa_1^2,\ldots,\kappa_p^2} \in \R^p$. Note that $\veckappa \geq 0$. Also note that for any $\vecmu$ and any $\vecx_1,\vecx_2 \in \X$, we have $K_\vecmu(\vecx_1,\vecx_2) \leq \ip{\vecmu}{\veckappa}$.

The notion of suitability used in \cite{two-stage} is that of \emph{kernel-goodness} first proposed in \cite{sim} for classification tasks. For sake of simplicity, we shall henceforth consider only binary classification tasks, the extension to multi-class classification tasks being straightforward. We present below the notion of goodness used in \cite{two-stage}. For any binary classification task over a domain $\X$ characterized by a distribution $\D$ over $\X \times \bc{\pm 1}$, a Mercer kernel $K : \X \times \X \rightarrow \R$ with associated Reproducing Kernel Hilbert Space $\H_K$ and feature map $\Phi_K : \X \rightarrow \H_K$ is said to be $(\epsilon, \gamma)$-\emph{kernel good} if there exists a unit norm vector $\vecw \in \H_K$ such that $\norm{\vecw}_{\H_K} = 1$ and the following holds
\[
\EE{(\vecx,y) \sim \D}{\bs{1 - \frac{y\ip{\vecw}{\Phi(\vecx)}}{\gamma}}_+} \leq \epsilon
\]

\section{Learning a \emph{Good} Kernel}
The key idea behind \cite{two-stage} is to try and learn a positive linear combination of kernels that is good according to the notion presented above. We define the risk functional $\RR(\cdot) : \R^p \mapsto \R^+$ as follows:
\[
\RR(\vecmu) := \EE{(\vecx,y),(\vecx',y') \sim \D \times \D}{\bs{1 - yy'K_\vecmu(\vecx,\vecx')}_+}
\]
A combination $\vecmu$ will be said to be $\epsilon$-\emph{combination good} if $\RR(\vecmu) \leq \epsilon$. The quantity $\RR(\vecmu)$ is of interest since an application of Jensen's inequality (see \cite[Lemma 3.2]{two-stage}) shows us that for any $\vecmu \geq 0$ that is $\epsilon$-combination good, the kernel $K_\vecmu$ is $\br{\epsilon,\frac{1}{\ip{\vecmu}{\veckappa}}}$-kernel good. Furthermore, one can show, using standard results on capacity of linear function classes (see for example \cite[Theorem 21]{rademacher}), that an $(\epsilon, \gamma)$-good kernel can be used to learn, with confidence $1 - \delta$, a classifier with expected misclassification rate at most $\epsilon + \epsilon_1$ by using at most $\O{\frac{\kappa^4}{\epsilon_1^2 \gamma^2}\log\frac{1}{\delta}}$ labeled samples.

In order to cast this learning problem more cleanly, \cite{two-stage} proposes the construction of a \emph{K-space} using the following feature map
\[
z : (\vecx,\vecx') \mapsto \br{K_1(\vecx,\vecx'),\ldots,K_p(\vecx,\vecx')} \in \R^p
\]
This allows us to write, for any $\vecmu \in \R^p$, $K_\vecmu(\vecx,\vecx') = \ip{\vecmu}{z(\vecx,\vecx')}$. Given $n$ labeled training points $(\vecx_1,y_1),\ldots,(\vecx_n,y_n)$, define the empirical risk functional $\hat\RR(\cdot) : \R^p \mapsto \R^+$ as follows\footnote{We note that \cite{two-stage} includes the terms $\bs{1 - \ip{\vecmu}{z(\vecx_i,\vecx_i)}}_+$ into the empirical risk as well. This does not change the asymptotics of our analysis except for causing a bit of notational annoyance. In order to account for this term, the true risk functional will have to include an additional term $\RR_{\text{add}}(\vecmu) := \EE{(\vecx,y) \sim \D}{\bs{1 - K_\vecmu(\vecx,\vecx)}_+}$. This will add a negligible term to the uniform convergence bound because we will have to consider the convergence of the term $\hat\RR_{\text{add}}(\vecmu) := \frac{2}{n(n+1)}\sum_{1 \leq i \leq n}\bs{1 - \ip{\vecmu}{z(\vecx_i,\vecx_i)}}_+$ to $\RR_{\text{add}}$. However, from thereon, the analysis will remain unaffected since $\RR_{\text{add}}(\vecmu) \geq 0$ so a combination $\vecmu$ having true risk $\RR(\vecmu) + \RR_{\text{add}}(\vecmu) \leq \epsilon$ will still give a kernel $K_\vecmu$ that is $\br{\epsilon,\frac{1}{\ip{\vecmu}{\veckappa}}}$-kernel good.}:
\[
\hat\RR(\vecmu) := \frac{2}{n(n-1)}\sum_{1 \leq i < j \leq n}\bs{1 - y_iy_j\ip{\vecmu}{z(\vecx_i,\vecx_j)}}_+
\]
\cite{two-stage} poses the learning problem as the following optimization problem:
\[
\underset{\vecmu \geq 0}{\min}\ \frac{\lambda}{2}\norm{\vecmu}_2^2 + \hat\RR(\vecmu)
\]
\section{Generalization Guarantees for a Learned Kernel Combination}
Our generalization guarantee shall proceed in two steps. We shall assume that we have with us a training set $(\vecx_1,y_1),\ldots,(\vecx_n,y_n)$ using which we are able to determine a combination vector $\hat\vecmu$ such that $\hat\RR(\hat\vecmu)\leq \hat\epsilon$.
\begin{enumerate}
	\item We shall first prove that, with high probability over the choice of the training points, the learned combination vector $\hat\vecmu$ will give us a kernel $K_{\hat\vecmu}$ that is $\br{\hat\epsilon + \epsilon_1, \frac{1}{\ip{\hat \vecmu}{\veckappa}}}$-kernel good where $\epsilon_1 > 0$ is a quantity that can be made arbitrarily small.
	\item We shall then prove that given that there exists a good combination of kernels in the K-space, with very high probability $\hat\epsilon$ will be very small. This we will prove by showing a converse of the inequality proved in the first step. This will allow us to give oracle inequalities for the kernel goodness of the learned combination.
\end{enumerate}

\subsection{Step 1}
In this step, we prove a uniform convergence guarantee for the learning problem at hand. Using standard proof techniques, we shall reduce the problem of uniform convergence to that of estimating the capacity of a certain function class. The notion of capacity we shall use is the Rademacher complexity which we shall bound using the heavy hammer of strong convexity based bounds from \cite{rad-bounds}. We note that the proof progression used in this step is fairly routine within the empirical process community and has been used to give generalization proofs for other problems as well (see for example \cite{ying-proof,u-stat}).

First of all we note that due to the optimization process we have\footnote{Any tolerance $\epsilon_{\text{opt}}$ offered by the optimizer can easily be incorporated into the bounds. However, we do not do so for sake of clarity.}
\[
\frac{\lambda}{2}\norm{\hat\vecmu}_2^2 \leq \frac{\lambda}{2}\norm{\hat\vecmu}_2^2 + \hat\RR(\hat\vecmu) \leq \frac{\lambda}{2}\norm{\veczero}_2^2 + \hat\RR(\veczero) = 1
\]
which implies that we need only concern ourselves with combination vectors inside the $L_2$ ball of radius $r_\lambda = \sqrt{\frac{2}{\lambda}}$.
\[
\B_2\br{r_\lambda} := \bc{\vecmu \in \R^p : \norm{\vecmu}_2 \leq r_\lambda}
\]
For notational simplicity, we denote $\vecz = (\vecx,y)$ as a training sample. For any training set $\vecz_1,\ldots,\vecz_n$ where $\vecz_i = (\vecx_i,y_i)$ and for any $\vecmu \in \R^p$, we write $\ell(\vecmu,\vecz_i,\vecz_j) := \bs{1 - y_iy_j\ip{\vecmu}{z(\vecx_i,\vecx_j)}}_+$. We assume, yet again for the sake of notational simplicity, that we obtain at all times, an even number of training samples i.e. $n$ is even. For a ghost sample $\tilde\vecz_1,\ldots,\tilde\vecz_n$ then, we can write
\begin{eqnarray*}
\E{\frac{2}{n(n-1)}\sum_{1 \leq i < j \leq n}\ell(\vecmu,\tilde\vecz_i,\tilde\vecz_j)} &=& \frac{2}{n(n-1)}\sum_{1 \leq i < j \leq n}\E{\ell(\vecmu,\tilde\vecz_i,\tilde\vecz_j)}\\
	&=& \frac{2}{n(n-1)}\sum_{1 \leq i < j \leq n}\RR(\hat\vecmu) = \RR(\hat\vecmu)
\end{eqnarray*}

Thus we can write
\begin{eqnarray*}
\RR(\hat\vecmu) - \hat\RR(\hat\vecmu) &=& \E{\frac{2}{n(n-1)}\sum_{1 \leq i < j \leq n}\ell(\vecmu,\tilde\vecz_i,\tilde\vecz_j)} - \hat\RR(\hat\vecmu)\\
									  &\leq& \underset{\vecmu \in \B_2(r_\lambda)}{\sup}\bc{\E{\frac{2}{n(n-1)}\sum_{1 \leq i < j \leq n}\ell(\vecmu,\tilde\vecz_i,\tilde\vecz_j)} - \hat\RR(\vecmu)}\\
\end{eqnarray*}
Let
\[
g\br{\vecz_1,\ldots,\vecz_n} = \frac{2}{n(n-1)}\underset{\vecmu \in \B_2(r_\lambda)}{\sup}\bc{\E{\sum_{1 \leq i < j \leq n}\ell(\vecmu,\tilde\vecz_i,\tilde\vecz_j)} - \sum_{1 \leq i < j \leq n}\ell(\vecmu,\vecz_i,\vecz_j)}
\]
For any $\vecmu \in \B_2(r_\lambda)$ and any $\vecx_1,\vecx_2 \in \X$, we have $K_\vecmu(\vecx_1,\vecx_2) \leq \ip{\vecmu}{\veckappa} \leq r_\lambda\norm{\veckappa}_2$. Using this, it is not difficult to see that the expression $g\br{\vecz_1,\ldots,\vecz_n}$ can be perturbed by at most $\frac{2}{n}\br{1 + r_\lambda\norm{\veckappa}_2}$ by the change of a single true training sample $\vecz_i = (\vecx_i,y_i)$ (see \cite[Theorem 3.4]{two-stage} for the calculations). Applying McDiarmid's inequality to this expression, we get with probability at least $1 - \delta$,
\[
\RR(\hat\vecmu) - \hat\RR(\hat\vecmu) \leq \E{g\br{\vecz_1,\ldots,\vecz_n}} + (1 + r_\lambda\norm{\veckappa}_2)\sqrt\frac{2\log\frac{1}{\delta}}{n}
\]
We now estimate the the expectation term on the right hand side.
\begin{eqnarray*}
\E{g\br{\vecz_1,\ldots,\vecz_n}} &=& \frac{2}{n(n-1)}\E{\underset{\vecmu \in \B_2(r_\lambda)}{\sup}\bc{\E{\sum_{1 \leq i < j \leq n}\ell(\vecmu,\tilde\vecz_i,\tilde\vecz_j)} - \sum_{1 \leq i < j \leq n}\ell(\vecmu,\vecz_i,\vecz_j)}}\\
								 &\leq& \frac{2}{n(n-1)}\E{\underset{\vecmu \in \B_2(r_\lambda)}{\sup}\bc{\sum_{1 \leq i < j \leq n}\ell(\vecmu,\tilde\vecz_i,\tilde\vecz_j) - \sum_{1 \leq i < j \leq n}\ell(\vecmu,\vecz_i,\vecz_j)}}
\end{eqnarray*}
We now invoke a powerful alternate representation for U-statistics to simplify the above expression. This method can be found in \cite{u-stat} that itself attributes this method to \cite{alternate-u-stat}. This, along with the Hoeffding decomposition, are two of the most powerful techniques to deal with ``coupled'' random variables as we have in this situation.
\begin{thm}[\cite{u-stat}, Lemma A.1]
For any set of real valued functions $q_\tau : \X \times \X \rightarrow \R$ indexed by $\tau \in T$, if $X_1,\ldots,X_n$ are i.i.d. random variables then we have
\[
\E{\underset{\tau \in T}{\sup} \frac{2}{n(n-1)}\sum_{1\leq i < j \leq n} q_\tau(X_i,X_j)} \leq \E{\underset{\tau \in T}{\sup} \frac{2}{n}\sum_{i = 1}^{n/2} q_\tau(X_i,X_{n/2 + i})}
\]
\end{thm}
Applying this decoupling result to the random variables $X_i = (\tilde\vecz_i, \vecz_i)$, the index set $\B_2(r_\lambda)$ and functions $q_\tau(X_i,X_j) = \ell(\vecmu,\tilde\vecz_i,\tilde\vecz_j) - \ell(\vecmu,\vecz_i,\vecz_j)$ we get
\begin{eqnarray*}
\E{g\br{\vecz_1,\ldots,\vecz_n}} &\leq& \frac{2}{n}\E{\underset{\vecmu \in \B_2(r_\lambda)}{\sup}\bc{\sum_{i = 1}^{n/2}\ell(\vecmu,\tilde\vecz_i,\tilde\vecz_{n/2 + i}) - \ell(\vecmu,\vecz_i,\vecz_{n/2 + i})}}\\
								&=& \frac{2}{n}\E{\underset{\vecmu \in \B_2(r_\lambda)}{\sup}\bc{\sum_{i = 1}^{n/2}\epsilon_i\br{\ell(\vecmu,\tilde\vecz_i,\tilde\vecz_{n/2 + i}) - \ell(\vecmu,\vecz_i,\vecz_{n/2 + i})}}}\\
								&\leq& \frac{4}{n}\E{\underset{\vecmu \in \B_2(r_\lambda)}{\sup}\bc{\sum_{i = 1}^{n/2}\epsilon_i\ell(\vecmu,\vecz_i,\vecz_{n/2 + i})}}\\
								&=& \frac{4}{n}\E{\underset{\vecmu \in \B_2(r_\lambda)}{\sup}\bc{\sum_{i = 1}^{n/2}\epsilon_i\bs{1 - y_iy_{n/2 + i}\ip{\vecmu}{z(\vecx_i,\vecx_{n/2 + i})}}_+}}\\
								&\leq& \frac{4}{n}\E{\underset{\vecmu \in \B_2(r_\lambda)}{\sup}\bc{\sum_{i = 1}^{n/2}\epsilon_i\ip{\vecmu}{z(\vecx_i,\vecx_{n/2 + i})}}} = 2\RR_{n/2}(\B_2(r_\lambda))\\
\end{eqnarray*}
where in the second step, we performed symmetrization on the decoupled expression by introducing Rademacher random variables $\epsilon_i, i = 1, \ldots, n/2$. In the fifth step we have applied the contraction inequality stated in Theorem~\ref{thm:contraction} below on the $1$-Lipschitz function $\phi_i : x \mapsto = \bs{1 - a_ix}_+$ where $a_i = y_iy_{n/2 + i}$. We have exploited the fact that Theorem~\ref{thm:contraction} actually proves the contraction inequality for the empirical Rademacher averages which allows us to treat $a_i$ as constants dependent only on $i$. 

\begin{thm}
\label{thm:contraction}
Let $\H$ be a set of bounded real valued functions from some domain $\X$ and let $x_1,\ldots,x_n$ be arbitrary elements from $\X$. Furthermore, let $\phi_i : \R \rightarrow \R$, $i = 1,\ldots,n$ be $L$-Lipschitz functions. Then we have
\[
\E{\underset{h \in \H}{\sup}\frac{1}{n}\sum_{i=1}^n \epsilon_i\phi_i(h(x_i))} \leq L \E{\underset{h \in \H}{\sup}\frac{1}{n}\sum_{i=1}^n \epsilon_ih(x_i)}
\]
\end{thm}
\begin{proof}
Ledoux and Talagrand (see \cite[Theorem 4.12]{talagrand}) prove the same result but for wrapper functions that satisfy $\phi_i(0) = 0$ for all $i$. To get the result, simply apply the result to the functions $\tilde\phi_i : x \mapsto \phi_i(x) - \phi_i(0)$ to get
\[
\E{\underset{h \in \H}{\sup}\frac{1}{n}\sum_{i=1}^n \epsilon_i\phi_i(h(x_i))} \leq \E{\underset{h \in \H}{\sup}\frac{1}{n}\sum_{i=1}^n \epsilon_i\tilde\phi_i(h(x_i))} + \E{\frac{1}{n}\sum_{i=1}^n \epsilon_i\phi_i(0)} \leq L \E{\underset{h \in \H}{\sup}\frac{1}{n}\sum_{i=1}^n \epsilon_ih(x_i)}
\]
where we apply \cite[Theorem 4.12]{talagrand} to the first term and the second term vanishes by linearity of expectation.
\end{proof}

The concluding term in the last chain of inequalities gives us the Rademacher complexity of the hypothesis class $\B_2(r_\lambda)$. At this point we introduce the following result on Rademacher complexities of regularized linear predictor classes
\begin{thm}[\cite{rad-bounds}, Theorem 1]
\label{thm:rad-bounds}
Let $\W$ be a closed convex set and let $F : \W \rightarrow \R$ be $\lambda$-strongly convex w.r.t. $\norm{\cdot}_\ast$. Assume $\W \subseteq \bc{\vecw : F(\vecw) \leq W_\ast^2}$. Furthermore, let $\X = \bc{\vecx : \norm{\vecx} \leq X}$ and $\F_\W := \bc{\vecw \mapsto \ip{\vecw}{\vecx} : \vecw \in \W, \vecx \in \X}$. Then, we have
\[
\hat\RR_n(\F_\W) \leq XW_\ast\sqrt{\frac{2}{\lambda n}}
\]
\end{thm}
Although \cite{rad-bounds} make their claim for the normal Rademacher average but their proof actually gives bounds for the empirical Rademacher averages. Since our hypothesis class is $L_2$ regularized, we can apply Theorem~\ref{thm:rad-bounds} to the $L_2$/$L_2$ case with $F(\vecmu) = \norm{\vecmu}_2^2$ as the regularizer. Since we have $\underset{\vecx_1,\vecx_2 \in \X}{\sup}{\norm{z(\vecx_1,\vecx_2)}_2} \leq \norm{\veckappa}_2$, we get
\[
\RR_{n/2}(\B_2(r_\lambda)) \leq r_\lambda\norm{\veckappa}_2\sqrt\frac{2}{n} = 2\norm{\veckappa}_2\sqrt\frac{1}{\lambda n}
\]
We have thus proved the following result
\begin{thm}
\label{thm:true-le-emp}
With probability at least $1 - \delta$ over the choice of training samples, the minimizer $\hat\vecmu$ of the expression
\[
\underset{\vecmu \geq 0}{\min}\ \frac{\lambda}{2}\norm{\vecmu}_2^2 + \hat\RR(\vecmu)
\]
satisfies the following
\[
\RR(\vecmu) \leq \hat\RR(\vecmu) + 4\norm{\veckappa}_2\sqrt\frac{1}{\lambda n} + \br{1 + \norm{\veckappa}_2\sqrt\frac{2}{\lambda}}\sqrt\frac{2\log\frac{1}{\delta}}{n} \leq \hat\RR(\vecmu) + 6\norm{\veckappa}_2\sqrt\frac{\log\frac{1}{\delta}}{\lambda n}
\]
\end{thm}
Since $\hat\vecmu \in \B_2(r_\lambda)$, we have $\ip{\hat\vecmu}{\veckappa} \leq \sqrt\frac{2}{\lambda}\norm{\veckappa}_2$. This implies that the kernel $K_{\hat\vecmu}$ is at least $\br{\hat\epsilon + \epsilon_1, \frac{1}{\norm{\veckappa}_2}\sqrt\frac{\lambda}{2}}$-kernel good where $\hat\epsilon = \hat\RR(\vecmu)$ and $\epsilon_1 \leq 6\norm{\veckappa}_2\sqrt\frac{\log\frac{1}{\delta}}{\lambda n}$. In particular, if all the $p$ kernels share a common bound i.e. $\kappa_i \leq \kappa$ for all $i$, then $\norm{\veckappa}_2 \leq \kappa^2\sqrt p$ and we can show the kernel $K_{\hat\vecmu}$ to be $\br{\hat\epsilon + 6\kappa^2\sqrt\frac{p\log\frac{1}{\delta}}{\lambda n}, \frac{1}{\kappa^2}\sqrt\frac{\lambda}{2p}}$-kernel good.

\subsection{Step 2}
Just as we analyzed the excess risk expression $\RR(\vecmu) - \hat\RR(\vecmu)$ uniformly over vectors the ball $\B_2(r_\lambda)$, we can similarly analyze the expression $\hat\RR(\vecmu) - \RR(\vecmu)$ uniformly over any (fixed) ball $\B_2(r)$ to get the following result.
\begin{thm}
\label{thm:emp-le-true}
Let $r > 0$ be some fixed radius, then with probability at least $1 - \delta$ over the choice of training samples, all combination vectors $\vecmu \in \B_2(r)$ satisfy
\[
\hat\RR(\vecmu) \leq \RR(\vecmu) + 2r\norm{\veckappa}_2\sqrt\frac{2}{n} + \br{1 + r\norm{\veckappa}_2}\sqrt\frac{2\log\frac{1}{\delta}}{n}
\]
\end{thm}
This allows us to give the following oracle inequality:
\begin{thm}
\label{thm:oracle-l2}
Suppose as an oracle assumption we assume that there exists a good combination vector $\vecmu_o$ that is $\epsilon_o$-combination good, then we can output with probability at least $1 - \delta$, for any $\epsilon_1 > 0$ using $n = \Om{\frac{\norm{\vecmu_o}_2^2}{\epsilon_1^3}}$ training samples, a combination vector such that the corresponding kernel that is $\br{\epsilon_o + \epsilon_1,\frac{1}{\norm{\veckappa}_2\norm{\vecmu_o}_2}\sqrt\frac{\epsilon_1}{3}}$-kernel good.
\end{thm}
\begin{proof}
Using Theorem~\ref{thm:emp-le-true} we have with probability at least $1 - \delta$,
\[
\hat\RR(\vecmu_o) \leq \epsilon_o + 2\norm{\vecmu_o}_2\norm{\veckappa}_2\sqrt\frac{2}{n} + \br{1 + \norm{\vecmu_o}_2\norm{\veckappa}_2}\sqrt\frac{2\log\frac{1}{\delta}}{n} \leq \epsilon_o + 6\norm{\vecmu_o}_2\norm{\veckappa}_2\sqrt\frac{2\log\frac{1}{\delta}}{n}
\]
Since $\hat\vecmu$ is the minimizer of the regularized empirical risk, we have
\[
\frac{\lambda}{2}\norm{\hat\vecmu}_2^2 + \hat\RR(\hat\vecmu) \leq \frac{\lambda}{2}\norm{\vecmu_o}_2^2 + \hat\RR(\vecmu_o) \leq \frac{\lambda}{2}\norm{\vecmu_o}_2^2 + \epsilon_o + 6\norm{\vecmu_o}_2\norm{\veckappa}_2\sqrt\frac{2\log\frac{1}{\delta}}{n}
\]
which gives us, since $\norm{\hat\vecmu}_2 \geq 0$,
\[
\hat\RR(\hat\vecmu) \leq \epsilon_o + \frac{\lambda}{2}\norm{\vecmu_o}_2^2 + 6\norm{\vecmu_o}_2\norm{\veckappa}_2\sqrt\frac{2\log\frac{1}{\delta}}{n}
\]
Applying Theorem~\ref{thm:true-le-emp} we get with probability at least $1 - 2\delta$,
\[
\RR(\hat\vecmu) \leq \epsilon_o + \frac{\lambda}{2}\norm{\vecmu_o}_2^2 + 6\norm{\vecmu_o}_2\norm{\veckappa}_2\sqrt\frac{2\log\frac{1}{\delta}}{n} + 6\norm{\veckappa}_2\sqrt\frac{\log\frac{1}{\delta}}{\lambda n}
\]
For any $0 < \epsilon_1 <3/4$, setting $\lambda = \frac{2\epsilon_1}{3\norm{\vecmu_o}_2^2}$ and requiring $n \geq \frac{500}{\epsilon_1^3}\norm{\vecmu_o}_2^2\norm{\veckappa}_2^2\log\frac{1}{\delta}$ so that all three terms in the above expression are less than $\epsilon_1/3$ gives us the result (for values of $\epsilon_1$ larger than $3/4$, $n \geq \frac{650}{\epsilon_1^2}\norm{\vecmu_o}_2^2\norm{\veckappa}_2^2\log\frac{1}{\delta}$ suffices).
\end{proof}
Such oracle inequalities are very desirable since they tell us that we would be able to give a performance that is competitive against any fixed kernel in foresight. If we set $\lambda$ to an oracle oblivious value such as $\lambda = \sqrt[3]{\frac{1}{n}}$ then although we get an inferior claim with respect to the kernel-goodness, we are able to make that claim in hindsight as well.

\section{Learning Sparse Kernel Combinations}
Since the complexity of the evaluating the kernel $K_\vecmu$ goes up roughly as $\norm{\vecmu}_0$, it is desirable to learn sparse combinations. This can be done by changing the learning formulation slightly to the following:
\[
\underset{\vecmu \geq 0}{\min} \frac{\lambda}{2}\norm{\vecmu}_1 + \frac{2}{n(n-1)}\sum_{1 \leq i < j \leq n}\bs{1 - y_iy_j\ip{\vecmu}{z(\vecx_i,\vecx_j)}}_+
\]
The above learning algorithm can also shown to admit generalization guarantees. For sake of brevity we only give below the main points where the analysis differs from the $L_2$ regularized case. First of all, we would be able to show that the regularized empirical risk minimizer $\hat\vecmu$ would lie in the $L_1$ ball $\B_1(s_\lambda) := \bc{\vecmu \in \R^p : \norm{\vecmu}_1 \leq s_\lambda}$ where $s_\lambda = \frac{2}{\lambda}$.

Due to this the perturbations to the expression $g\br{\vecz_1,\ldots,\vecz_n}$ would be limited by $\frac{2}{n}\br{1 + s_\lambda\norm{\veckappa}_\infty}$. While applying Theorem~\ref{thm:rad-bounds}, we would instead consider the regularizer $F(\vecmu) = \norm{\vecmu}_q^2$ for $q = \frac{\log p}{\log p - 1}$ which is $\br{\frac{1}{\log p}}$-strongly convex with respect to the norm $\norm{\cdot}_1$. This would allow us to bound the Rademacher complexity of the hypothesis class $\B_1(s_\lambda)$ as
\[
\RR_{n/2}\br{\B_1(s_\lambda)} \leq s_\lambda\norm{\veckappa}_\infty\sqrt\frac{2\log p}{n} = \frac{2\norm{\veckappa}_\infty}{\lambda}\sqrt\frac{2\log p}{n}
\]
This allows us to make the following claim:
\begin{thm}
\label{thm:true-le-emp-sparse}
With probability at least $1 - \delta$ over the choice of training samples, the minimizer $\hat\vecmu$ of the expression
\[
\underset{\vecmu \geq 0}{\min}\ \frac{\lambda}{2}\norm{\vecmu}_1 + \hat\RR(\vecmu)
\]
satisfies the following
\[
\RR(\vecmu) \leq \hat\RR(\vecmu) + \frac{4\norm{\veckappa}_\infty}{\lambda}\sqrt\frac{2\log p}{n} + \br{1 + \frac{2\norm{\veckappa}_\infty}{\lambda}}\sqrt\frac{2\log\frac{1}{\delta}}{n} \leq \hat\RR(\vecmu) + \frac{6\norm{\veckappa}_\infty}{\lambda\sqrt{n}}\br{\sqrt{\log p} + \sqrt{\log 1/\delta}}
\]
\end{thm}
Since $\hat\vecmu \in \B_1(s_\lambda)$, we have $\ip{\hat\vecmu}{\veckappa} \leq \frac{2\norm{\veckappa}_\infty}{\lambda}$. This implies that the kernel $K_{\hat\vecmu}$ is $\br{\hat\epsilon + \epsilon_1, \frac{\lambda}{2\norm{\veckappa}_\infty}}$-kernel good where $\hat\epsilon = \hat\RR(\vecmu)$ and $\epsilon_1 \leq \frac{6\norm{\veckappa}_\infty}{\lambda\sqrt{n}}\br{\sqrt{\log p} + \sqrt{\log 1/\delta}}$. In particular, if all the $p$ kernels share a common bound i.e. $\kappa_i \leq \kappa$ for all $i$, then $\norm{\veckappa}_\infty \leq \kappa^2$ and we can show the kernel $K_{\hat\vecmu}$ to be $\br{\hat\epsilon + \frac{6\kappa^2}{\lambda\sqrt{n}}\br{\sqrt{\log p} + \sqrt{\log 1/\delta}}, \frac{\lambda}{2\kappa^2}}$-kernel good.

Note that this result has a much better dependence on $p$ that the result for $L_2$ regularized learning where we were only able to show that the kernel $K_{\hat\vecmu}$ was $\br{\hat\epsilon + 6\kappa^2\sqrt\frac{p\log\frac{1}{\delta}}{\lambda n}, \frac{1}{\kappa^2}\sqrt\frac{\lambda}{2p}}$-kernel good.

We can also show the following version of Theorem~\ref{thm:emp-le-true} to be true
\begin{thm}
\label{thm:emp-le-true-sparse}
Let $s > 0$ be some fixed radius, then with probability at least $1 - \delta$ over the choice of training samples, all combination vectors $\vecmu \in \B_1(s)$ satisfy
\[
\hat\RR(\vecmu) \leq \RR(\vecmu) + 2s\norm{\veckappa}_\infty\sqrt\frac{2\log p}{n} + \br{1 + s\norm{\veckappa}_\infty}\sqrt\frac{2\log\frac{1}{\delta}}{n}
\]
\end{thm}
Using this, and going as before, we are also able to guarantee the following oracle inequality similar to Theorem~\ref{thm:oracle-l2}
\begin{thm}
\label{thm:oracle-l1}
Suppose as an oracle assumption we assume that there exists a good combination vector $\vecmu_o$ that is $\epsilon_o$-combination good, then we can output with probability at least $1 - \delta$, for any $\epsilon_1 > 0$ using $n = \Om{\frac{\norm{\vecmu_o}_1}{\epsilon_1^2}}$ training samples, a combination vector such that the corresponding kernel that is $\br{\epsilon_o + \epsilon_1,\frac{\epsilon_1}{3\norm{\veckappa}_\infty\norm{\vecmu_o}_1}}$-kernel good.
\end{thm}
\begin{proof}
Following the chain of inequalities given by Theorems~\ref{thm:true-le-emp-sparse} and \ref{thm:emp-le-true-sparse} and using optimality of the regularized empirical risk minimizer $\hat\vecmu$, we get, with probability at least $1 - 2\delta$,
\[
\RR(\hat\vecmu) \leq \epsilon_o + \frac{\lambda}{2}\norm{\vecmu_o}_1 + \frac{6\norm{\veckappa}_\infty\br{\sqrt{\log p} + \sqrt{\log 1/\delta}}}{\sqrt n}\br{\norm{\vecmu_o}_1 + \frac{1}{\lambda}}
\]
Setting $\lambda = \frac{2 \epsilon_1}{3\norm{\vecmu_o}_1}$ and requiring $n \geq \frac{135\norm{\vecmu_o}_1\norm{\veckappa}_\infty\br{\sqrt{\log p} + \sqrt{\log 1/\delta}}}{\epsilon_1^2}$ finishes the proof.
\end{proof}

\section{Discussion on the Nature of Guarantees}
The guarantees given above, both for the sparse as well as the non-sparse kernel learning cases are slightly unsatisfactory in the sense they assume combination goodness to ensure kernel goodness. In other words they assume the existence of a combination that is $\epsilon$-combination before guaranteeing that the output would be a kernel that is $(\epsilon', \gamma')$-kernel good. Ideally, we should have used the promise of existence of a kernel that $(\epsilon, \gamma)$-kernel good to ensure that a good kernel is output.

One way to prove such a result would be to show that if there exists a kernel combination that is $\br{\epsilon, \gamma}$-kernel good, then there also exists some combination $\vecmu \in \R^p$ that is $\epsilon'$-combination good for some $\epsilon' > 0$. However, this is an unlikely result and the the aim of this section is to discuss this point. It turns out that the biggest hurdle that one faces in proving such a result is the form of combination goodness chosen by \cite{two-stage}. The definition of combination goodness used in \cite{two-stage} is related to the notion of \emph{similarity goodness} proposed in \cite{sim} except for the absence of a \emph{weight function}.

More specifically, \cite{sim} consider a kernel $K_\vecmu$ to be $\epsilon$-similarity good if for some weight function $w : \X \rightarrow \R$ the following holds:
\[
\EE{(\vecx,y) \sim \D}{\bs{1 - y\EE{(\vecx',y') \sim \D}{y'w(\vecx')K_\vecmu(\vecx,\vecx')}}_+} \leq \epsilon
\]
For ease of comparison, we have absorbed the margin parameter $\gamma$ in the definition given in \cite{sim} into the weight function $w(\cdot)$. Note that if the notion of combination goodness had been defined using
\[
\RR(\vecmu) := \EE{(\vecx,y),(\vecx',y') \sim \D \times \D}{\bs{1 - yy'w(\vecx')K_\vecmu(\vecx,\vecx')}_+}
\]
instead, then one could have used some form of inverse Jensen's inequality to convert similarity goodness into combination goodness. Since the presence of the weight function makes it possible for crisp conversions of kernel goodness into similarity goodness as was done in \cite{kernel-sim-compare}, this could have been one way to convert kernel goodness into combination goodness (i.e. via similarity goodness). However, due to the absence of such weight functions, it seems difficult to convert kernel goodness into combination goodness using the methods of \cite{kernel-sim-compare}. 

Another reason to believe in the non-existence of such conversions from kernel to combination goodness is the form of the predictor in the RKHS. If one looks at the proof of Lemma 3.2 in \cite{two-stage} then one notices that the kernel goodness is proven with respect to the predictor $\vecw = \EE{(\vecx,y) \sim \D}{y'\Phi_{\H_{K_\vecmu}}(\vecx)}$ where $\Phi_{\H_{K_\vecmu}} : \X \mapsto \H_{K_\vecmu}$ is the feature map corresponding to the kernel $K_\vecmu$. This turns out to be very a restrictive form for the predictor. A kernel can be good due to the existence of \emph{any} unit norm predictor in its RKHS. However the notion of combination goodness seems to prefer predictors that point from the mean of the images of the negative points to the mean of the images of the positive points in the RHKS. It was noted in \cite{sim} that such a notion of goodness is too strong (\cite{sim} actually call this the \emph{strongly-good} notion of similarity goodness) and that there exist kernels that are very good with respect to the learning task at hand but the uniform vectors $\vecw = \EE{(\vecx,y) \sim \D}{y'\Phi_{\H_{K_\vecmu}}(\vecx)}$ in their RKHSes perform poorly (see \cite[Definition 2]{sim} and the discussion thereafter).

Thus it seems unlikely that the current proof technique can be extended to accept promises of kernel goodness. The technique seems inherently suited to accept combination goodness and output good kernels. It would be interesting to see whether the existing proofs can be modified or whether the algorithms can be modified to accommodate kernel goodness.

\bibliographystyle{plain}
\bibliography{ref}
\end{document}